\documentclass{ifacconf}
\usepackage{array}
\usepackage{graphicx}      
\usepackage{natbib}        
\usepackage{subcaption}
\usepackage{amssymb}
\usepackage{amsmath}       
\usepackage{xcolor}

\begin{document}
\begin{frontmatter}

\title{Robust Adaptive Backstepping Impedance Control of Robots in Unknown Environments}


\author[First]{Reza Nazmara}
\author[Second]{Alap Kshirsagar}
\author[Second]{Jan Peters}
\author[First]{A. Pedro Aguiar}

\address[First]{Research Center for Systems and Technologies (SYSTEC), ARISE, Faculty of Engineering, University of Porto, 4200-465 Porto, Portugal. (email: up202103272@edu.fe.up.pt; pedro.aguiar@fe.up.pt)}
\address[Second]{Intelligent Autonomous Systems Lab, Department of Computer Science, TU Darmstadt, Germany. (email: alap@robot-learning.de; jan.peters@tu-darmstadt.de)}

\begin{abstract}                
This paper presents a Robust Adaptive Backstepping Impedance Control (RABIC) strategy for robots operating in contact-rich and uncertain environments. The proposed control strategy considers the complete coupled dynamics of the system and explicitly accounts for key sources of uncertainty, including external disturbances and unmodeled dynamics, while not requiring the robot’s dynamic parameters in implementation.  
 We propose a backstepping-based adaptive impedance control scheme for the inner loop to track the reference impedance model. To handle uncertainties, we employ a Taylor series–based estimator for system dynamics, and an adaptive estimator for determining the upper bound of external forces. Stability analysis demonstrates the semi-global practical finite-time stability of the overall system. To demonstrate the effectiveness of the proposed method, a simulated mobile manipulator scenario and experimental evaluations on a real Franka Emika Panda robot were conducted. The proposed approach exhibits safer performance compared to PD control while ensuring trajectory tracking and force monitoring.
 Overall, the RABIC framework provides a solid basis for future research on adaptive and learning-based impedance control for coupled mobile and fixed serially linked manipulators.
\end{abstract}
\end{frontmatter}



\section{Introduction}
\vspace{-2mm}
Impedance control is a well-established interaction control strategy that enables compliant and safe robot--environment interactions~\citep{hogan1985impedance}, with broad applications in industrial automation and medical robotics~\citep{kong2025neural}. Unlike conventional tracking controllers that primarily minimize trajectory errors~\citep{zhang2024adaptive, yan2024observer}, impedance control regulates the dynamic relationship between motion and external forces~\citep{hogan1985impedance, nazmara2020exponentially}, thereby improving safety, robustness, and adaptability in contact-rich and uncertain environments.
Since robots often operate at high speeds in environments with unknown dynamics, unintended collisions are inevitable~\citep{nazmara2020exponentially}. Such events can be costly and hazardous, particularly in industrial settings involving valuable equipment and human operators. Impedance control plays a critical role in mitigating these risks by shaping the robot’s interaction behavior without relying on cameras or explicit environment modeling. By adjusting mechanical impedance in joint or task space, robots can absorb disturbances, handle uncertainties, and limit excessive contact forces. \\
\begin{figure*}[t!]
\begin{center}
\includegraphics[width=1\linewidth]{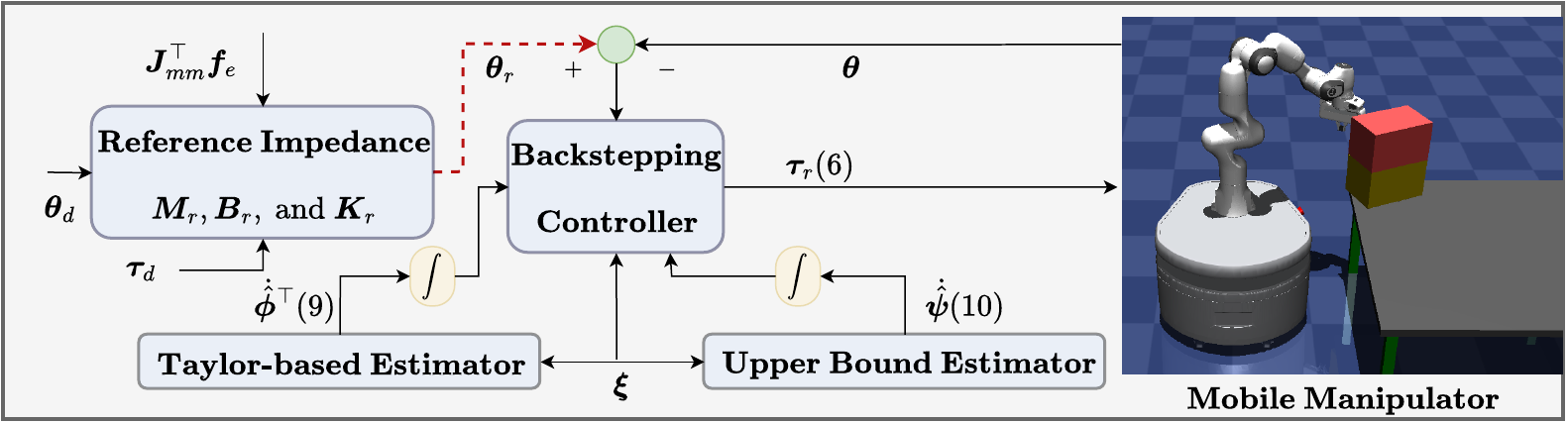}    
\caption{Overall block diagram of the proposed control strategy. A model reference impedance defines the desired inner-loop behavior. A Taylor series–based estimator and a force upper bound estimator in the backstepping control handle uncertainties and manage unwanted collisions in the WCMM's and fixed-base manipulators' joint space.}

\label{Fig:01}
\end{center}
\end{figure*}
Wheeled mobile manipulators (WMMs)~\citep{li2007adaptive, zhou2022coupled}, which integrate a robotic arm with a mobile base, can offer significant advantages over stationary robotic arms~\citep{zhang2024adaptive, ahmadi2023adaptive} or mobile platforms without a manipulator~\citep{yan2024observer, dacs2025robust}. Combining both mobility and manipulation capabilities in a single system extends the operational workspace, increases flexibility, and enables the execution of complex tasks that are unattainable for fixed-base manipulators or mobile platforms alone. In collaborative settings, wheeled coupled mobile manipulators~\citep{zhou2022coupled} are capable of working together to manipulate large or unwieldy objects, execute synchronized assembly tasks, and adapt to changing environments. \\
Most research on impedance control has primarily focused on fixed-base manipulators with relatively few studies addressing mobile manipulators, and even fewer control approaches that can be implemented on both mobile and fixed-base manipulators without modifying the control structure, while accounting for the robot's configuration and dynamics. For example, \citet{chien2004adaptive} proposed an adaptive impedance control approach using function approximation in a regressor-free design. Similarly, \citet{zhang2018master} developed a master--slave adaptive neural network framework for underwater manipulators, and \citet{kong2025neural} introduced a neural network--based optimal impedance control method for soft environmental interactions. \citet{huo2021model} presented a model-free adaptive impedance strategy for autonomous sanding, while \citet{zhu2025active} proposed an active-impedance-based adaptive locomotion method for bionic hexapod robots.
Although studies on mobile manipulators remain scarce, existing implementations often rely on simplifying assumptions. For instance, in~\citep{souzanchi2017robust}, a SCARA manipulator was considered, with negligible third-joint coupling, and impedance control was applied only along \(z\). Other common simplifications include modeling the environment as a linear spring or assuming zero desired inertia, which restrict\textcolor{blue}{s} the flexibility and generality of the impedance behavior.
An impedance control strategy that operates independently of the robot’s dynamic parameters, enabling seamless deployment on both coupled mobile manipulators and fixed‑base manipulators, represents a practical and versatile design with straightforward implementation. \\
 In comparison to the impedance control methods discussed above, our proposed approach is simpler in both design and implementation, and can be deployed on mobile as well as fixed-base manipulators. More specifically, our method adopts a torque control strategy within the impedance control framework. Unlike impedance control schemes based on voltage control~\citep{souzanchi2017robust,nazmara2020exponentially, ahmadi2023adaptive}, the torque-based approach preserves a direct correspondence with the system’s physical dynamics. This results in improved accuracy, enhanced robustness against external disturbances, and superior trajectory tracking performance. \\
Incorporating direct force feedback within the inner control loop can lead to instability and difficulties during collisions. Unlike previous approaches~\citep{chien2004adaptive, zhang2018master, huo2021model}, our proposed method does not rely on a force sensor directly integrated into the torque control loop and does not require the nominal values of the robot’s dynamic parameters. Furthermore, compared to the regressor-based design presented in~\citep{zhang2018master}, the proposed approach offers greater
implementation flexibility
 for robots with different configurations. \\

Unlike learning-based or RL-based control architectures, which often rely on offline training, extensive datasets, or computationally intensive policy updates, the proposed approach performs closed-form parameter estimation fully online, without prior data, and with significantly lower computational requirements. This distinguishes our framework from model-based actor--critic methods~\citep{zhao2022model}, whose dependence on accurate system models and substantial training effort may limit their suitability for real-time implementation in contact-rich scenarios. \\
Recent RL formulations have demonstrated notable adaptability in dynamic and uncertain environments, yet managing uncertainty in a stable and computationally efficient manner can remain challenging. For instance, the model-based actor--critic strategy presented in~\cite{yao2023model} ensures asymptotic Lyapunov stability under soft disturbance assumptions, which are well suited to the hydraulic systems studied therein but less aligned with the abrupt and nonsmooth forces characteristic of impedance-control settings. Data-driven RL approaches~\cite{yao2024data} provide improved flexibility for handling complex uncertainties; however, their training demands, exploration requirements, and safety considerations may introduce practical difficulties when deploying them on physical robotic platforms with strict real-time constraints. \\
In contrast, the proposed robust adaptive controller avoids trial-and-error learning, guarantees immediate online adaptation, and ensures practically finite-time stability, offering advantages in safety-critical interaction tasks. These properties enable its application to high-DOF robotic systems while maintaining predictable behavior during collisions and interactions with unstructured environments. \\
Furthermore, the Taylor-series-based estimator employed in our framework can be interpreted as a lightweight local model approximation of the unknown Euler--Lagrange dynamics. This viewpoint relates our approach to recent model--data hybrid strategies~\cite{yao2024modeldata}, which integrate analytical dynamics with learning-based components to improve modeling accuracy. However, unlike hybrid architectures that typically require neural network training or data accumulation, the proposed estimator operates entirely online, preserves a linear-in-parameters structure, and maintains formal Lyapunov-based guarantees. These characteristics make the method particularly suitable for real-time and safety-critical robotic applications, where computational efficiency, robustness, and immediate adaptability are essential.

In this context, developing robust impedance control strategies for high-degree-of-freedom (DoF) robots is essential to fully exploit their capabilities while ensuring safe, reliable, and high-performance operation in real-world environments. The main novelties and contributions of the proposed method are summarized as follows:

\begin{enumerate}
    \item Developed a robust adaptive impedance control framework for high-DOF robots, enabling operation in both joint and task spaces and effective interaction with unknown environments.
    \item Combined backstepping control with Lyapunov analysis to ensure semi-global finite-time stability and robust adaptive performance.
    \item Validated the approach through comprehensive simulations and experiments on robotic systems.
\end{enumerate}

The rest of the paper is organized as follows. Section~\ref{sec:Dynamics} presents the robot dynamics. Section~\ref{sec:Proposed} introduces the proposed control scheme and its stability analysis. Section~\ref{sec:Simulation & Experiments} reports simulation and experimental results. Finally, Section~\ref{sec:Conclusion} concludes the paper and outlines future work.
\section{Dynamic Modeling of Robotic Systems}\label{sec:Dynamics}
\indent We present the robot's dynamic equation in a general configuration, i.e., a collaborative mobile manipulator~\citep{li2007adaptive,zhou2022coupled}: 

\begin{equation}\label{Eq:01}
\boldsymbol{D}(\boldsymbol{\theta})\,\ddot{\boldsymbol{\theta}}
+ \boldsymbol{C}(\boldsymbol{\theta}, \dot{\boldsymbol{\theta}})\,\dot{\boldsymbol{\theta}}
+ \boldsymbol{G}(\boldsymbol{\theta}) + \boldsymbol{J}_{mm}^\top \boldsymbol{f}_{\mathrm{e}}
= \boldsymbol{\tau}_r + \boldsymbol{\tau}_{u}. 
\end{equation}
In \eqref{Eq:01}, \(\boldsymbol{D}(\boldsymbol{\theta}) \in \mathbb{R}^{n \times n}\) denotes the inertia matrix, 
\(\boldsymbol{C}(\boldsymbol{\theta}, \dot{\boldsymbol{\theta}}) \in \mathbb{R}^{n \times n}\) represents the Coriolis 
and centrifugal matrix, and \(\boldsymbol{G}(\boldsymbol{\theta}) \in \mathbb{R}^{n}\) is the gravity vector. 
Here, \(n\) is the total number of joints in the WCMM system, defined as \(n = n_b + n_m\), 
where \(n_b\) and \(n_m\) denote the number of joints in the mobile base and manipulator, respectively. The vectors \(\boldsymbol{\tau}_r \in \mathbb{R}^{n}\) and \(\boldsymbol{\tau}_{u} \in \mathbb{R}^{n}\) represent the joint control effort and the lumped uncertainties, which include unmodeled dynamics,  external disturbances, and friction. The term \(\boldsymbol{J}_{mm} \in \mathbb{R}^{6 \times n}\) denotes the full coupled Jacobian matrix of the WCMM, 
obtained by concatenating the Jacobian of the mobile base, \(\boldsymbol{J}_b \in \mathbb{R}^{6 \times n_b}\), 
and the Jacobian of the manipulator, \(\boldsymbol{J}_m \in \mathbb{R}^{6 \times n_m}\), 
such that \(\boldsymbol{J}_{mm} = [\boldsymbol{J}_b \ \boldsymbol{J}_m]\). The vector \(\boldsymbol{f}_{\mathrm{e}} \in \mathbb{R}^{6}\) represents the external force acting at the robot end-effector. 
In this context, \(\boldsymbol{\theta}\) denotes the vector of joint angular positions, defined as 
\(\boldsymbol{\theta} = \left[\boldsymbol{\theta}_b \ \boldsymbol{\theta}_m \right]^\top\), 
where \(\boldsymbol{\theta}_b\) and \(\boldsymbol{\theta}_m\) correspond to the angular positions of the mobile base and manipulator, respectively. 
The components of these vectors are given by 
\(\boldsymbol{\theta}_b = \left[\theta_R \ \theta_L \right]^\top \in \mathbb{R}^{2 \times 1}\), 
where \(\theta_R\) and \(\theta_L\) denote the angular positions of the right and left wheels, respectively. 
Similarly, 
\(\boldsymbol{\theta}_m = \left[\theta_1 \ \dots \ \theta_i \ \dots \ \theta_{n_m} \right]^\top \in \mathbb{R}^{n_m \times 1}\), 
where \(\theta_i\) represents the angular position of the \(i\)-th manipulator joint. 
For clarity, all matrices and vectors in this paper are denoted in boldface.

 Considering the variable \(\boldsymbol{\theta}_r\) as the output of the reference impedance model and defining the tracking error as \(\boldsymbol{e} = \boldsymbol{\theta}_r - \boldsymbol{\theta}\), the following auxiliary variables are introduced: \(\boldsymbol{\xi}_1 = \int (\dot{\boldsymbol{e}} + \mu \boldsymbol{e}) \, \mathrm{d}\tau\) and \(\boldsymbol{\xi}_2 = \dot{\boldsymbol{e}} + \mu \boldsymbol{e}\), where \(\mu\) is a positive definite design parameter. Based on these definitions, the following state-space representation can be derived
\begin{subequations}
\label{Eq:02}
\begin{align}
\dot{\boldsymbol{\xi}}_1 &={\boldsymbol{\xi}}_2, \label{eq:2A}\\
\Dot{\boldsymbol{\xi}}_2 &=-\hat{\boldsymbol{D}}^{-1}\boldsymbol{\tau}_r
 + \boldsymbol{H}. \label{eq:2B}
\end{align}
\end{subequations}
where \(\hat{\boldsymbol{D}}\) denotes the known nominal part of \(\boldsymbol{D}\), and \(\boldsymbol{H}\) represents the uncertainty function, defined as

 \begin{multline}
\label{Eq:03}
\boldsymbol{H} =
\ddot{\boldsymbol{\theta}}_r 
+ \mu \dot{\boldsymbol{e}} -\left(\boldsymbol{D}^{-1} - \hat{\boldsymbol{D}}^{-1}\right)\boldsymbol{\tau}_r \\
- \boldsymbol{D}^{-1} 
\left[ 
 \boldsymbol{\tau}_{u} 
- \boldsymbol{C} \dot{\boldsymbol{\theta}}
- \boldsymbol{G} 
- \boldsymbol{J}_{mm}^\top \boldsymbol{f}_{\mathrm{e}}
\right].
\end{multline}
Equation~\eqref{Eq:02} presents the new state-space representation of the overall tracking dynamics of the system, which will be used in the next section to derive the proposed impedance control scheme.

\section{Proposed Control  and Stability}\label{sec:Proposed}
In this section, we propose a low-level impedance control strategy based on a model-reference impedance control approach. 
\subsection{Proposed Control}
In this subsection, we present the proposed impedance control strategy for robotic systems.
 The control effort is designed to track the reference joint position \(\boldsymbol{\theta}_r\), which is the output of the reference impedance model, while compensating for system uncertainties and external disturbances. The reference impedance model dynamics, which generate the reference joint position signal \(\boldsymbol{\theta}_r\), are described by the following equation
 \begin{multline}
\label{Eq:04}
\ddot{\boldsymbol{\theta}}_r =
\boldsymbol{M}_r^{-1} \left( \boldsymbol{\tau}_{d} - \boldsymbol{J}_{mm}^\top \boldsymbol{f}_{\mathrm{e}} \right)
+ \ddot{\boldsymbol{\theta}}_{d} \\
- \boldsymbol{M}_r^{-1} \boldsymbol{B}_r \left( \dot{\boldsymbol{\theta}}_r - \dot{\boldsymbol{\theta}}_d \right)
- \boldsymbol{M}_r^{-1} \boldsymbol{K}_r \left( \boldsymbol{\theta}_r - \boldsymbol{\theta}_d \right),
\end{multline}

where \(\boldsymbol{M}_r \in \mathbb{R}^{n \times n}\), \(\boldsymbol{B}_r \in \mathbb{R}^{n \times n}\), and \(\boldsymbol{K}_r \in \mathbb{R}^{n \times n}\) are diagonal matrices representing the desired inertia, damping, and stiffness, respectively. The signal \(\boldsymbol{\theta}_d \in \mathbb{R}^{n}\) denotes the desired joint trajectory, and \(\boldsymbol{\tau}_{d} \in \mathbb{R}^{n}\) represents the desired joint torque for the mobile manipulator.

Before proposing the inner-loop control effort, we model the uncertainty function \(\boldsymbol{H}\) as follows
\begin{equation}
\label{Eq:05}
\boldsymbol{H} = \boldsymbol{\gamma}^\top \boldsymbol{\phi} + \boldsymbol{\vartheta},
\end{equation}
where \(\boldsymbol{\gamma}\) is the regressor, \(\boldsymbol{\phi}\) is the vector of actual parameters, and \(\boldsymbol{\vartheta}\) represents the unmodeled dynamics, which are assumed to be upper-bounded as \(|\vartheta_i| \le \psi_i\). The boundedness assumption~\citep{ahmadi2023adaptive} on the unmodeled dynamics is realistic and consistent with the simulation and experimental results, based on the physical interpretation of serially linked robot arms and the formulation in~\eqref{Eq:03}. We clarify that the uncertainty term $\boldsymbol{H}$ aggregates modeling errors, actuator limits, and interaction forces, all of which are bounded due to finite link masses, joint limits, and feasible contact forces. Consequently, with bounded reference trajectories, all closed-loop signals remain bounded, justifying the standard boundedness assumption used in the stability analysis.

\textbf{Remark 1:} For all expressions of the form $x^{q}$ with $0 < q < 1$ and $x$ possibly negative, the implementation uses the signed power function: $\operatorname{sign}(x)\,|x|^{q}$, as is standard for real-valued control inputs (see e.g., \citep{bhat2000finite}).

Then, to track the desired impedance command \(\boldsymbol{\theta}_r\), we propose the following control for each joint by taking into account Remark~1 in implementation as
\begin{multline}
\label{Eq:06}
\tau_{ri} = 
\hat{D}_i \Big[\xi_{1i} + k_{1i}\left(2l-1\right)\xi_{2i}\xi_{1i}^{2(l-1)}\Big]+\hat{D}_i k_{2i}\\\left(\xi_{2i} + k_{1i}\xi_{1i}^{2l-1}\right)^{2l-1} 
+ \hat{D}_i\Big[\hat{\tau}_i
+ \, \operatorname{sign}(\xi_{2i} + k_{1i} \xi_{1i}^{2l-1})\hat{\psi}_i  \Big],
\end{multline}


 where \(k_{1i}\) and \(k_{2i}\) are positive control gains, \(\hat{\psi}_i\) denotes the estimate of \(\psi_i\), \(l\) satisfies \(0 < l < 1\), and \(\hat{\tau}_i\) represents the output of the Taylor-based estimator~\citep{ahmadi2023adaptive}, defined by the following equation
\begin{equation}
\label{Eq:07}
\hat{\tau}_i =\boldsymbol{\gamma}_i^{\top}(t) \hat{\boldsymbol{\phi}}_i(t),
\end{equation}

where \(\hat{\boldsymbol{\phi}}_i\) is the estimate of \(\boldsymbol{\phi}_i\) and \(\boldsymbol{\gamma}_i\) is defined as
\begin{equation}
\label{Eq:08}
\boldsymbol{\gamma}_i(t)\textcolor{blue}{^\top} = 
\begin{bmatrix}
1, \;
\left( \int_{0}^{t} (\xi_{2i}(\tau) - \xi_{2i0}) d\tau \right)^{k}, \;
(\xi_{2i}(t) - \xi_{2i0})^{m}
\end{bmatrix},
\end{equation}
in which \( k = 1, \ldots, l_1 \) and \( m = 1, \ldots, l_2 \).
At the reference point \( \xi_{2i0} \), the combined \( l_1 \)-th and \( l_2 \)-th order Taylor polynomial terms are represented by \( \hat{\tau}_i \), which can be parameterized in a linear form as \( \hat{\tau}_i = \boldsymbol{\gamma}_i(t) \hat{\boldsymbol{\phi}}_i(t) \). Here, \( \hat{\boldsymbol{\phi}}_i(t) \in \mathbb{R}^{(l_1 + l_2 + 1) \times 1} \) denotes the vector of Taylor coefficients for \( H_i(t) \), and \( \boldsymbol{\gamma}_i(t)\textcolor{blue}{^\top} \in \mathbb{R}^{(l_1 + l_2 + 1) \times 1} \) is the regressor vector as previously defined.

 Then, we propose the parameter adaptation laws using the following equations
\begin{equation}\label{Eq:09}
\dot{\hat{\phi}}_i^{\top} = \rho_{\phi_i} \, (\xi_{2i} + k_{1i} \xi_{1i}^{2l-1})\boldsymbol{\gamma}_i^{\top} - \sigma_{\phi_i} \, \hat{\boldsymbol{\phi}}_i^{\top},
\end{equation}

\begin{equation}\label{Eq:10}
\dot{\hat{\psi}}_i = \rho_{\psi_i} \, \big\lvert \xi_{2i} + k_{1i} \xi_{1i}^{2l-1} \big\rvert - \sigma_{\psi_i} \, \hat{\psi}_i,
\end{equation}
in which \(\rho_{\phi_i}\), \(\sigma_{\phi_i}\), \(\rho_{\psi_i}\), and \(\sigma_{\psi_i}\) are positive constants.

The overall block diagram of the system is depicted in Figure~\ref{Fig:01}. As shown in this figure, the desired impedance command \(\boldsymbol{\theta}_r\) is first generated by a reference model and then provided to the inner loop, where a backstepping controller is used. The parameters within this backstepping controller are adjusted using two adaptive units. This backstepping controller is derived based on Lyapunov's theorem to ensure the stability of the system.

\subsection{Stability Analysis}
In this subsection, we analyze the stability of the proposed control scheme using Lyapunov theory. To formalize the stability properties, we present the following theorem, which is subsequently proven.
\begin{thm}
Consider the system described by~\eqref{Eq:01}, controlled according to the law in~\eqref{Eq:06}, with the Taylor-based estimator designed in~\eqref{Eq:07} and adaptation parameters specified by~\eqref{Eq:09} and~\eqref{Eq:10}, and with modeling uncertainty~$\boldsymbol{H}$ as in~\eqref{Eq:05}. Suppose the desired trajectory~$\boldsymbol{\theta}_d$ is sufficiently smooth, so that its time derivatives exist up to the required order, and assume that the unmodeled dynamics error~$\boldsymbol{\vartheta}$ is upper bounded. By substituting the control effort and adaptation parameters, letting~$\boldsymbol{\xi}_1$ and~$\boldsymbol{\xi}_2$ be the Taylor bases, and applying the system dynamics~\eqref{Eq:02} together with the reference model~\eqref{Eq:04}, it can be established that the closed-loop system is semi-globally, practically, finite-time stable.
\end{thm}

\begin{pf}
The proof proceeds by defining the following Lyapunov candidate function for the second subsystem using the backstepping technique
\begin{multline}
\label{Eq:11} 
 V_i = 0.5\left(\xi_{1i}\right)^2 + 0.5\left(\xi_{2i} + k_{1i}\xi_{1i}^{2l-1}\right)^2 \\+ 0.5\left( \tilde{\boldsymbol{\phi}}_i^\top\tilde{\boldsymbol{\phi}}_i/\rho_{{\phi}_i}+\tilde{\psi}_i^2/\rho_{{\psi}_i}\right)    
\end{multline} 
 By differentiating~\eqref{Eq:11}, substituting the control effort from~\eqref{Eq:06}, modeling the uncertainty as in~\eqref{Eq:05}, and utilizing the system dynamics in~\eqref{Eq:02},  we obtain
 \begin{multline}
\label{Eq:12}
\dot{V}_i = \xi_{1i}\xi_{2i} - \left(\xi_{2i} + k_{1i}\xi_{1i}^{2l-1}\right)\xi_{1i} - k_{2i}\left(\xi_{2i} + k_{1i}\xi_{1i}^{2l-1}\right)^{2l}\\
-\left(\xi_{2i} + k_{1i}\xi_{1i}^{2l-1}\right)\left(  \boldsymbol{\gamma}_i^{\top} \hat{\boldsymbol{\phi}}_i 
+ \hat{\psi}_i \, \operatorname{sign}(\xi_{2i} + k_{1i} \xi_{1i}^{2l-1}) \right) \\+ \left(\xi_{2i} + k_{1i}\xi_{1i}^{2l-1}\right)\left(\boldsymbol{\gamma}_i^\top \boldsymbol{\phi}_i + \vartheta_i\right) - \dot{\hat{\boldsymbol{\phi}}}_i^\top\tilde{\boldsymbol{\phi}}_i/\rho_{{\phi}_i}-\dot{\hat{\psi}}_i\tilde{\psi}_i/\rho_{{\psi}_i}
\end{multline}

 By defining \(\tilde{\boldsymbol{\phi}}_i = \boldsymbol{\phi}_i - \hat{\boldsymbol{\phi}}_i\) and \(\tilde{\psi}_i = \psi_i - \hat{\psi}_i\), and under the assumption that \(|\vartheta_i| \leq \psi_i\), after some manipulations and simplifications, we can rewrite~\eqref{Eq:12} as follows
 \begin{multline}
\label{Eq:13}
\dot{V}_i\leq -k_{1i}\xi_{1i}^{2l}-k_{2i}\left(\xi_{2i} + k_{1i}\xi_{1i}^{2l-1}\right)^{2l} + \left(\xi_{2i} + k_{1i}\xi_{1i}^{2l-1}\right)\boldsymbol{\gamma}_i^\top \tilde{\boldsymbol{\phi}}_i\\
|\xi_{2i} + k_{1i}\xi_{1i}^{2l-1}|\tilde{\psi}_i -  \dot{\hat{\boldsymbol{\phi}}}_i^\top\tilde{\boldsymbol{\phi}}_i/\rho_{{\phi}_i}-\dot{\hat{\psi}}_i\tilde{\psi}_i/\rho_{{\psi}_i}
\end{multline}
 
 By using the adaptation laws given in~\eqref{Eq:09} and~\eqref{Eq:10}, we can rewrite~\eqref{Eq:13} as the following inequality
 \begin{multline}
\label{Eq:14}
\dot{V}_i\leq -k_{1i}\xi_{1i}^{2l}-k_{2i}\left(\xi_{2i} + k_{1i}\xi_{1i}^{2l-1}\right)^{2l} + \left(\xi_{2i} + k_{1i}\xi_{1i}^{2l-1}\right)\boldsymbol{\gamma}_i^\top \tilde{\boldsymbol{\phi}}_i\\
|\xi_{2i} + k_{1i}\xi_{1i}^{2l-1}|\tilde{\psi}_i -  \left(\rho_{\phi_i} \, (\xi_{2i} + k_{1i} \xi_{1i}^{2l-1})\boldsymbol{\gamma}_i^{\top} - \sigma_{\phi_i} \, \hat{\boldsymbol{\phi}}_i^{\top}\right)\tilde{\boldsymbol{\phi}}_i/\rho_{\phi_i}\\-\left(\rho_{\psi_i} \, \big\lvert \xi_{2i} + k_{1i} \xi_{1i}^{2l-1} \big\rvert - \sigma_{\psi_i} \, \hat{\psi}_i\right)\tilde{\psi}_i/\rho_{{\psi}}
\end{multline}

Then, after some simplifications, the inequality~\eqref{Eq:14} becomes
 \begin{multline}
\label{Eq:15}
\dot{V}_i\leq -k_{1i}\xi_{1i}^{2l}-k_{2i}\left(\xi_{2i} + k_{1i}\xi_{1i}^{2l-1}\right)^{2l} + \left(\sigma_{\phi_i}/\rho_{\phi_i}\right)\hat{\boldsymbol{\phi}}_i^{\top}\tilde{\boldsymbol{\phi}}_i \\ + \left(\sigma_{\psi_i}/\rho_{\psi_i}\right)\hat{\psi}_i\tilde{\psi}_i
\end{multline}
We then employ the following inequalities
\begin{equation}\label{Eq:16}
\hat{\boldsymbol{\phi}}_i^{\top}\tilde{\boldsymbol{\phi}}_i \leq 0.5\boldsymbol{\phi}_i^{\top}\boldsymbol{\phi}_i - 0.5   \tilde{\boldsymbol{\phi}}_i^{\top}\tilde{\boldsymbol{\phi}}_i
\end{equation}
\begin{equation}\label{Eq:17}
\hat{\psi}_i\tilde{\psi}_i \leq 0.5 \psi_i^2-0.5\tilde{\psi}_i^2
\end{equation}
By substituting \eqref{Eq:16} and \eqref{Eq:17} into \eqref{Eq:15} we have
 \begin{multline}
\label{Eq:18}
\dot{V}_i\leq -k_{1i}\xi_{1i}^{2l}-k_{2i}\left(\xi_{2i} + k_{1i}\xi_{1i}^{2l-1}\right)^{2l} - \left(\sigma_{\phi_i}/2\rho_{\phi_i}\right)\tilde{\boldsymbol{\phi}}_i^{\top}\tilde{\boldsymbol{\phi}}_i  \\- \left(\sigma_{\psi_i}/2\rho_{\psi_i}\right)\tilde{\psi}_i^2 + \left(\sigma_{\phi_i}/2\rho_{\phi_i}\right)\boldsymbol{\phi}_i^{\top}\boldsymbol{\phi}_i +  \left(\sigma_{\psi_i}/2\rho_{\psi_i}\right)\psi_i^2
\end{multline}
By applying Lemma~\ref{lemma:inequality2} from Appendix~\ref{app:lemmas} and considering \(0 < l < 1\), with the substitutions \(a = 1 - l\), \(b = l\), \(q_1 = 1\), \(q_2 = \tilde{\phi}_{ij}^2\), and \(p = l^{\frac{l}{1-l}}\), we deduce that
 \begin{equation}\label{Eq:19}
- \tilde{\boldsymbol{\phi}}_i^{\top}\tilde{\boldsymbol{\phi}}_i\leq \left(1-l\right)p-\left(\tilde{\boldsymbol{\phi}}_i^{\top}\tilde{\boldsymbol{\phi}}_i\right)^{l}
 \end{equation}
Similarly, the following result can also be obtained
 \begin{equation}\label{Eq:20}
- \tilde{\psi}_i^2\leq \left(1-l\right)p-\left(\tilde{\psi}_i^2\right)^{l} 
 \end{equation}

Then we choose \(\rho_i\) as the following equation
 \begin{equation}\label{Eq:21}
\rho_i=\mathrm{min}\left\{2k_{1i},2k_{2i},\sigma_{\phi i},\sigma_{\psi i}\right\}
 \end{equation}
Finally, by applying Lemma~\ref{lemma:inequality3} from Appendix~\ref{app:lemmas}, together with \eqref{Eq:19}, \eqref{Eq:20}, and \eqref{Eq:21}, inequality \eqref{Eq:18} can be reformulated as follows
\begin{equation}
\label{Eq:22}
\dot{V}_i + \rho_i V_i^{l_i} \leq c_i
\end{equation}
 where \(c_i\) is defined as the following
\begin{multline}
\label{Eq:23}
c_i =\left(\sigma_{\phi_i}/\left(2\rho_{\phi_i}\right) + \sigma_{\psi_i}/\left(2\rho_{\psi_i}\right)\right)\left(1-l\right)p + \\\left(\sigma_{\phi_i}/2\rho_{\phi_i}\right)\boldsymbol{\phi}_i^{\top}\boldsymbol{\phi}_i +  \left(\sigma_{\psi_i}/2\rho_{\psi_i}\right)\psi_i^2
\end{multline}
 
 Inequality~\eqref{Eq:22} implies that the system described in \eqref{Eq:01} is semi-globally practically finite-time stable~\citep{nazmara2024safe}, 
and the tracking error converges to the desired impedance model defined in \eqref{Eq:04}.
\end{pf}
\section{Simulation and Experimental Results}
\label{sec:Simulation & Experiments}
This section presents simulations and experiments validating the effectiveness, robustness, and safety of the proposed control scheme, where joint-space impedance control ensures full-body robot protection.
\subsection{Simulation Setup and Baseline Scenarios on WMMs}
Here, we designed two scenarios to evaluate the system’s performance under collision, representative of situations commonly encountered by mobile manipulators operating in dynamic environments. These scenarios involve collisions with a stiff object with different masses, where the equilibrium point—defined as the contact point when contact forces begin to appear— may change, thereby making the interaction behavior more realistic. 
Specifically, a mobile manipulator with nine DoF interacts with an environment consisting of a red box placed on top of a yellow box on a table (see Fig.~\ref{Fig:01}). The mass of the red box is chosen between \(m_r = 2.5~\mathrm{kg}\), \(5~\mathrm{kg}\), and \(20~\mathrm{kg}\), and the friction coefficient is \(0.9\). The mass of the yellow box \(m_y\) is \(25~\mathrm{kg}\) and it is fixed to the table. 

We compare the performance of a conventional PD control scheme (Scenario~A) with our robust adaptive impedance control scheme (Scenario~B). Our control approach is applied simultaneously to all joints of the mobile manipulator, allowing the coupling effects—particularly those arising after collisions—to be effectively monitored and analyzed to ensure the overall safety of the robot.

\emph{Scenario A: PD Control on WMMs in MuJoCo}

In this scenario, the robot interacts with the environment while operating under a conventional PD control scheme. 
The objective is to examine the system behavior when trajectory tracking is performed without any awareness of contact forces or collisions. 
Under PD control, the robot strictly follows the predefined desired trajectory, which is intentionally designed to include an obstacle, resulting in an unintended collision. 
Upon impact, the PD controller interprets the resulting contact force as an external disturbance and continues attempting to follow the trajectory. 
The proportional and derivative gains are defined as 
\(\boldsymbol{k}_p = \begin{bmatrix}
60 & 60 & 480 & 600 & 480 & 480 & 420 & 480 & 28.8
\end{bmatrix}^{\top}\) and 
\(\boldsymbol{k}_d = \begin{bmatrix}
24 & 24 & 192 & 96 & 192 & 192 & 192 & 156 & 14.4
\end{bmatrix}^{\top}\), respectively. Figures~\ref{Fig:02}(A.1)–(A.3) illustrate the state of the environment after the collision for three different red-box masses, along with the corresponding contact force for each case. In Scenarios~A.1 and~A.2, the red box falls due to the impact, whereas in Scenario~A.3, the box is pushed; however, the robot lacks sufficient actuation power to sustain the motion and overcome friction. Overall, such situations, in which the robot experiences unexpected collisions without a dedicated interaction-control strategy, can significantly compromise performance and even lead to mechanical damage.

\begin{figure*}[ht]
\centering
\begin{minipage}[t]{0.3\textwidth}
  \centering
  \textbf{Scenario A.1}
\end{minipage}%
\begin{minipage}[t]{0.3\textwidth}
  \centering
  \textbf{Scenario A.2}
\end{minipage}%
\begin{minipage}[t]{0.3\textwidth}
  \centering
  \textbf{Scenario A.3}
\end{minipage}%

\vspace{-0.1em}

\begin{minipage}[t]{0.3\textwidth}
  \centering
  \includegraphics[width=\linewidth,height=3cm]{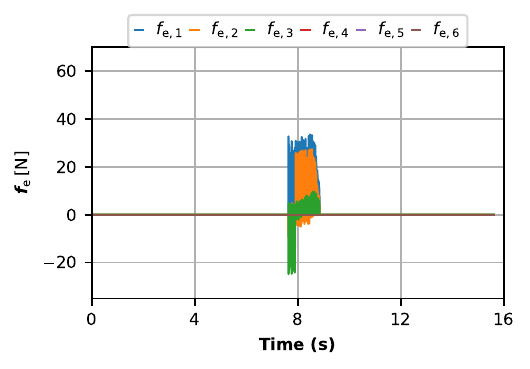}\\[0.5em]
  
\includegraphics[width=0.7\linewidth,height=3cm]{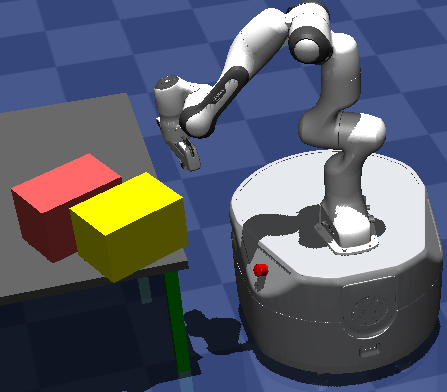}

\end{minipage}%
\begin{minipage}[t]{0.3\textwidth}
  \centering
  \includegraphics[width=\linewidth,height=3.05cm]{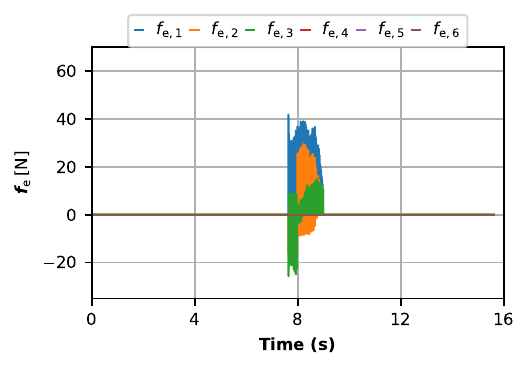}\\[0.5em]
  \includegraphics[width=0.7\linewidth,height=3cm]{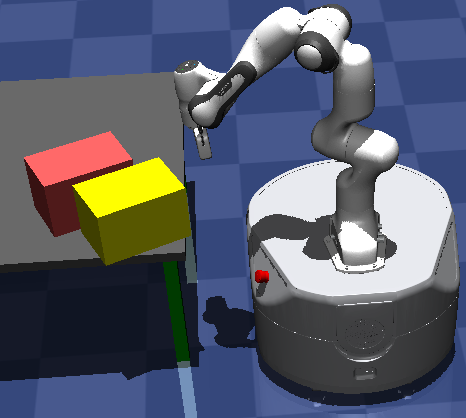}
\end{minipage}%
\begin{minipage}[t]{0.3\textwidth}
  \centering
  \includegraphics[width=\linewidth,height=3cm]{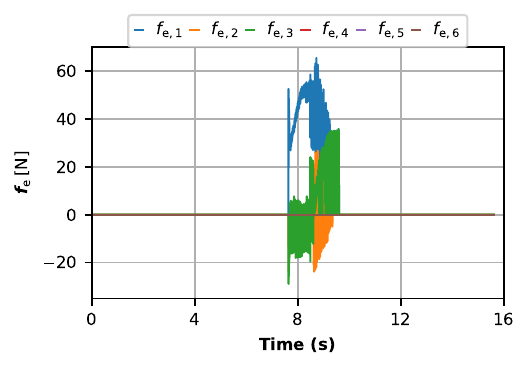}\\[0.5em]
  \includegraphics[width=0.7\linewidth,height=3cm]{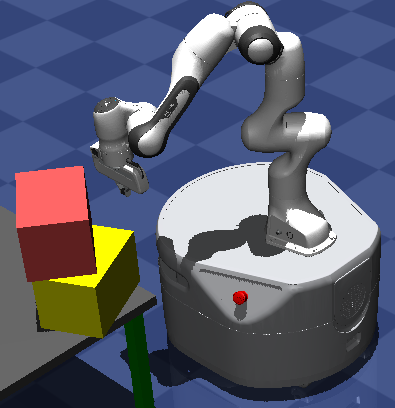}
\end{minipage}

\caption{\textbf{(A.1)--(A.3):} Snapshots and force profiles for collisions under PD control with a red box (\(m_r = 2.5,~5,~20~\mathrm{kg}\)). The robot, unaware of contact, follows its trajectory, causing impacts that knock over or push the box. The contact force increases with heavier boxes, highlighting the risks of damage and instability without interaction-aware control. For lighter boxes, PD control pushes the box until it falls, while the 20 kg box moves without knocking over.}

\label{Fig:02}
\end{figure*}

\emph{Scenario B: Proposed RABIC on WMMs in MuJoCo}

In this scenario, our proposed robust adaptive impedance control scheme is implemented across all joint spaces.  
Figures~\ref{Fig:03}(B.1)–(B.3) illustrate the state of the environment after the collision for three different red-box masses, similar to Scenario-A, along with the corresponding force vectors.
\begin{figure*}[ht]
\centering
\begin{minipage}[t]{0.3\textwidth}
  \centering
  \textbf{Scenario B.1}
\end{minipage}%
\begin{minipage}[t]{0.3\textwidth}
  \centering
  \textbf{Scenario B.2}
\end{minipage}%
\begin{minipage}[t]{0.3\textwidth}
  \centering
  \textbf{Scenario B.3}
\end{minipage}%

\vspace{-0.1em}

\begin{minipage}[t]{0.30\textwidth}
  \centering
  \includegraphics[width=\linewidth,height=3cm]{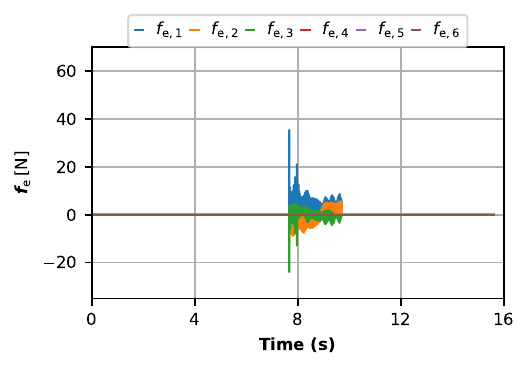}\\[0.5em]
  \includegraphics[width=0.7\linewidth,height=3cm]{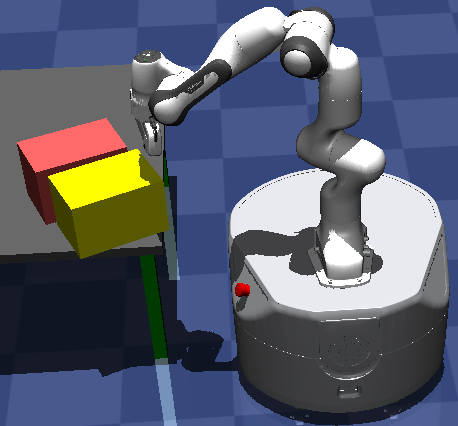}
\end{minipage}%
\begin{minipage}[t]{0.3\textwidth}
  \centering
  \includegraphics[width=\linewidth,height=3cm]{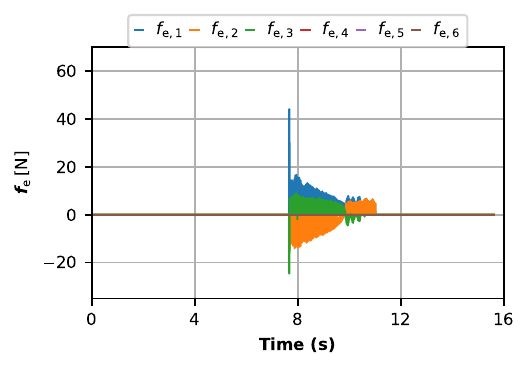}\\[0.5em]
  \includegraphics[width=0.7\linewidth,height=3cm]{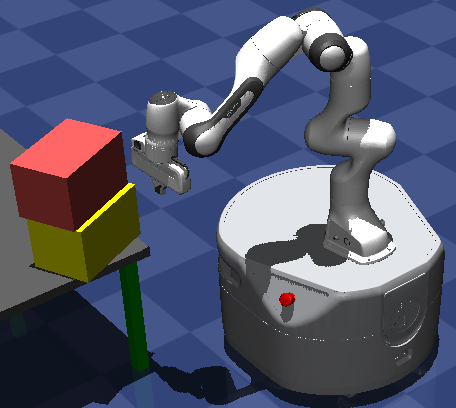}
\end{minipage}%
\begin{minipage}[t]{0.3\textwidth}
  \centering
  \includegraphics[width=\linewidth,height=3cm]{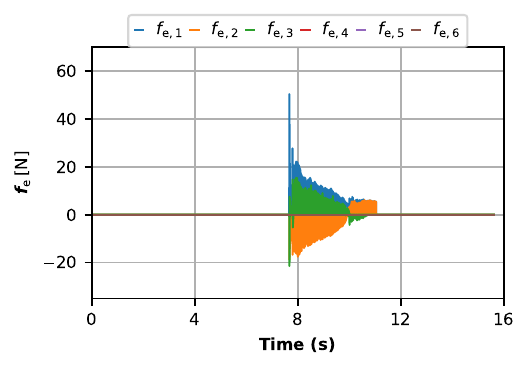}\\[0.5em]
  \includegraphics[width=0.7\linewidth,height=3cm]{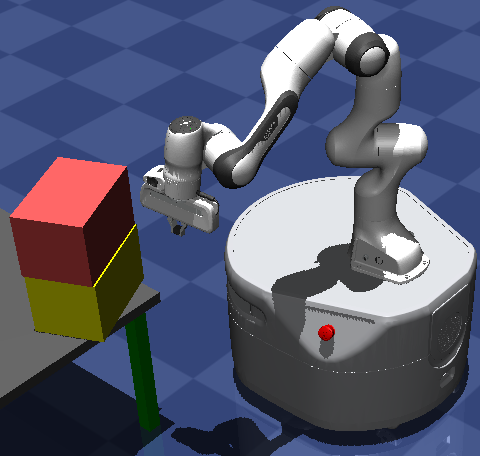}
\end{minipage}

\caption{\textbf{(B.1)--(B.3):} Snapshots and force profiles for three collision scenarios under the proposed adaptive impedance control, where the robot impacts a red box (\(m_r = 2.5,~5,~20~\mathrm{kg}\)). Trajectory adaptation in response to collisions rapidly attenuates contact forces. For lighter boxes, the system behaves similarly to a PD controller, while for heavier boxes, it limits force growth by relaxing trajectory tracking to ensure safety and stability.}
\label{Fig:03}
\end{figure*}
Our control scheme results in an initial contact force peak due to the sudden impact, similar to the PD control case but with a significantly smaller subsequent force.  
After the collision, the impedance controller adjusts the robot's motion by following the reference impedance signal, \(\boldsymbol{\theta}_r\),  
temporarily deviating from the desired trajectory \(\boldsymbol{\theta}_d\), to generate joint torques in directions that reduce the contact forces.  
As a result, the contact forces rapidly decrease and eventually approach zero. The signal \(\boldsymbol{\theta}_r\) serves as a reference trajectory that not only conveys time-dependent motion information  
but also encodes the desired mechanical impedance of the robot, thereby improving the safety of physical interactions. In both scenarios, the desired trajectory is given by \( \theta_{d,1}(t) = \theta_{d,2}(t) = 8 \sin\left( 0.2\pi t/t_f \right)\), \\ \(\theta_{d,3}(t) = 1.2 \sin\left( 0.3\pi t/t_f \right)\), and the rest of the joints are held at constant positions \(\theta_{d,4}(t) = 0\), \(\theta_{d,5}(t) = 1.7\), \(\theta_{d,6}(t) = -2.35\), \(\theta_{d,7}(t) = 0\), \(\theta_{d,8}(t) = 2.35\), and \(\theta_{d,9}(t) = 0.23\) for \(t_f = 2.5\). The timestep of the simulation is $1ms$.

The desired impedance parameters are set to \(\boldsymbol{M}_r = \boldsymbol{I}_n\), \(\boldsymbol{B}_r = 20\boldsymbol{I}_n\), \(\boldsymbol{K}_r = \boldsymbol{I}_n\), and \(\boldsymbol{\tau}_d = 0\). The control gains are chosen as \(\boldsymbol{k}_1 = \operatorname{diag}(5.76,\, 5.76,\, 24,\, 3.6,\, 24,\, 24,\, 24,\, 24,\, 24)\) and \(\boldsymbol{k}_2 = \operatorname{diag}(7.7,\, 7.7,\, 32,\, 4.8,\, 32,\, 32,\, 32,\, 32,\, 32)\). The adaptive parameters are chosen as \(\rho_\phi = 50\), \(\rho_\psi = 0.1\), and \(\sigma_\phi = \sigma_\psi = 0.005\), while the knowledge of the nominal inertia is assumed to be very limited, set as \(\hat{\boldsymbol{D}} = \boldsymbol{I}_n\). These parameters are determined through Lyapunov stability guidelines to achieve an appropriate balance among minimizing tracking error, maintaining smooth control effort, and enabling a rapid response of the impedance controller to unexpected collision events. 

A brief guideline for tuning the control parameters is as follows. The impedance parameters $(M_r, B_r, K_r)$ are first selected using the Routh--Hurwitz criterion to shape the desired behavior in~\eqref{Eq:04} and to ensure safe inertia, damping, and stiffness during collisions. Smaller $K_r$ values increase compliance but may degrade tracking, while larger values reduce sensitivity to contact forces and make the behavior closer to PD control.
The backstepping gains $k_1$ and $k_2$ are then chosen empirically, starting from small positive values and increasing them until they dominate the impedance dynamics and ensure contraction of $\xi_1$ and $\xi_2$. This procedure follows the Lyapunov conditions in Section~3 and allows a wide range of admissible gains. Finally, the adaptive rates $\rho_{\phi_i}, \rho_{\psi_i}, \sigma_{\phi_i}, \sigma_{\psi_i}$ are tuned to regulate adaptation speed and robust bound estimation, adjusted by observing transient and steady-state behavior. The real-robot experiments confirm that the controller remains robust across a broad set of parameter choices.
Figure~\ref{Fig:04} shows the control effort in Scenario~B.3, where the red box is heaviest (\(20~\mathrm{kg}\)), increasing response demands and instability risk. Even after strong collisions, the control remains smooth, while adaptive parameters adjust continuously.

\subsection{Real-Robot Experimental Validation on the FR3 Robot}
\begin{figure}[t]
\begin{center}
\includegraphics[width=0.8\linewidth]{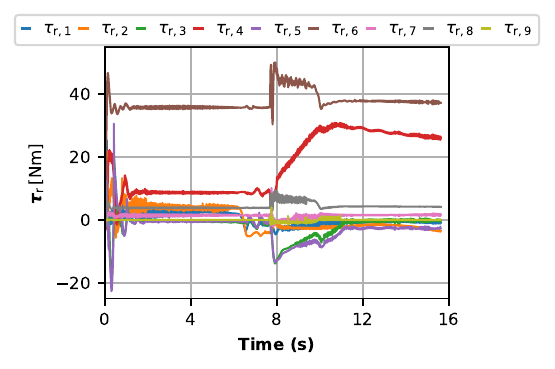}    
\caption{Joint control efforts \(\boldsymbol{\tau}_r\) under the proposed robust adaptive impedance control for Scenario~B.3.}
\label{Fig:04}
\end{center}
\end{figure}
Both the baseline PD controller and the proposed impedance control were implemented on an FR3 Robot tracking reference trajectories during obstacle collisions.
The experiments used MushroomRL~\citep{deramo2021mushroomrl} and PyLibfranka, with bias gravity compensation activated.
 Scenario~C uses PD control, while Scenario~D uses the proposed method. The desired joint-space trajectory is \( \theta_{d,1}(t) = 0.7 \sin(\omega t)\\\left(1 - \exp(-\omega t)(1 + \omega t)\right) \), where \( \omega = 2\pi/20 \), and the remaining desired joint positions are set to \( \theta_{d,2}(t) = 0.38 \), \( \theta_{d,3}(t) = 1.01 \), \( \theta_{d,4}(t) = -2.3 \), \( \theta_{d,5}(t) = -0.51 \), \( \theta_{d,6}(t) = 2.45 \), and \( \theta_{d,7}(t) = 1.17 \). The obstacle has a total mass of \(1.766\,\mathrm{kg}\), as shown in the real robot task space in Figure~\ref{Fig:05}.

\emph{Scenario C: PD Control on a Real Robot}

The results, including collision torques at the joints, control effort, and inner-loop tracking performance, are shown in Figures~\ref{Fig:06}(C.1.a)–(C.1.c), respectively.
PD gains are:\\ \(\boldsymbol{k}_p = \begin{bmatrix}
31.2 & 9.6 & 9.6 & 16.8 & 16.8 & 16.8 & 1.15
\end{bmatrix}^{\top}\) and\\
\(\boldsymbol{k}_d = \begin{bmatrix}
1.92 & 0.96 & 0.96 & 0.96 & 1.92 & 1.92 & 0.14
\end{bmatrix}^{\top}\).
 According to Figure~\ref{Fig:06}(C.1.a)), the peak impact torque is \(-5.025\,\mathrm{Nm}\) at joint~1 (\(t = 13.2\,\mathrm{s}\)). The results indicate that the PD controller ignores contact forces, causing the torque to increase again after an initial decrease to nearly \(-4.53\,\mathrm{Nm}\).

\emph{Scenario D: Proposed RABIC on a Real Robot}

The results for this scenario are shown in Figures \ref{Fig:06}(D.1.a)–(D.1.c), which correspond to the external collision-induced joint torques, control effort, and inner-loop tracking behavior, respectively. According to Figure~\ref{Fig:06}(D.1.a), the peak collision-induced joint torque is \(-4.578\,\mathrm{Nm}\) at \(t = 15.7\,\mathrm{s}\). Thereafter, the external torques decrease below \(-3.25\,\mathrm{Nm}\) to prevent excessive environmental force. This reduction enhances safety, as the robot can no longer push the box as it did under PD control (Figure~\ref{Fig:06}(C.1.a)). The proposed control efforts (Figure~\ref{Fig:06}(D.1.b)) exhibit more active dynamic response compared to PD control (Figure~\ref{Fig:06}(C.1.b)), while remaining soft enough to be implemented on the robot joints and maintaining small inner-loop tracking errors (Figure~\ref{Fig:06}(D.1.c)), unlike PD control where tracking after collision is more difficult (Figure~\ref{Fig:06}(C.1.c)).
  The real-robot torque profiles exhibit reduced chattering compared to simulation, as indicated by the lower root-mean-square (RMS) of the torque rate (real: 0.0065, simulation: 0.2477). This reduction is mainly attributed to inherent mechanical damping, sensor filtering, and the limited control bandwidth of the physical system, which attenuate high-frequency components.
  The remaining control gains are:
\(\boldsymbol{M}_r = \boldsymbol{I}_n\), \(\boldsymbol{B}_r = \operatorname{diag}(6.32,\, 20,\, 20,\, 35,\, 20,\, 20,\, 20)\), \(\boldsymbol{K}_r = \operatorname{diag}(10,\, 1,\, 1,\, 300,\, 1,\, 1,\, 1)\), \(\boldsymbol{k}_1 = \operatorname{diag}(0.48,\, 0.84,\, 0.48, \, 0.84,\\ \, 0.48,\, 0.48,\, 0.48)\), and \(\boldsymbol{k}_2 = \operatorname{diag}(0.64,\, 1.12, \, 0.64,\, 1.12,\, 0.64, \\ \, 0.64,\, 0.64)\). The adaptive parameters are chosen as \(\rho_\phi = \operatorname{diag}(4,\, 20,\, 2,\, 20,\, 2,\, 2,\, 2)\), \(\rho_\psi = \operatorname{diag}(0.05,\, 0.1,\, 0.05,\, 0.5, \\ \, 0.05,\, 0.05,\, 0.05)\), \(\sigma_\phi = \sigma_\psi = 0.001\), and \(l = 0.999\).

\begin{figure}[t]
\begin{center}
\includegraphics[width=0.45\linewidth,height=3.6cm]{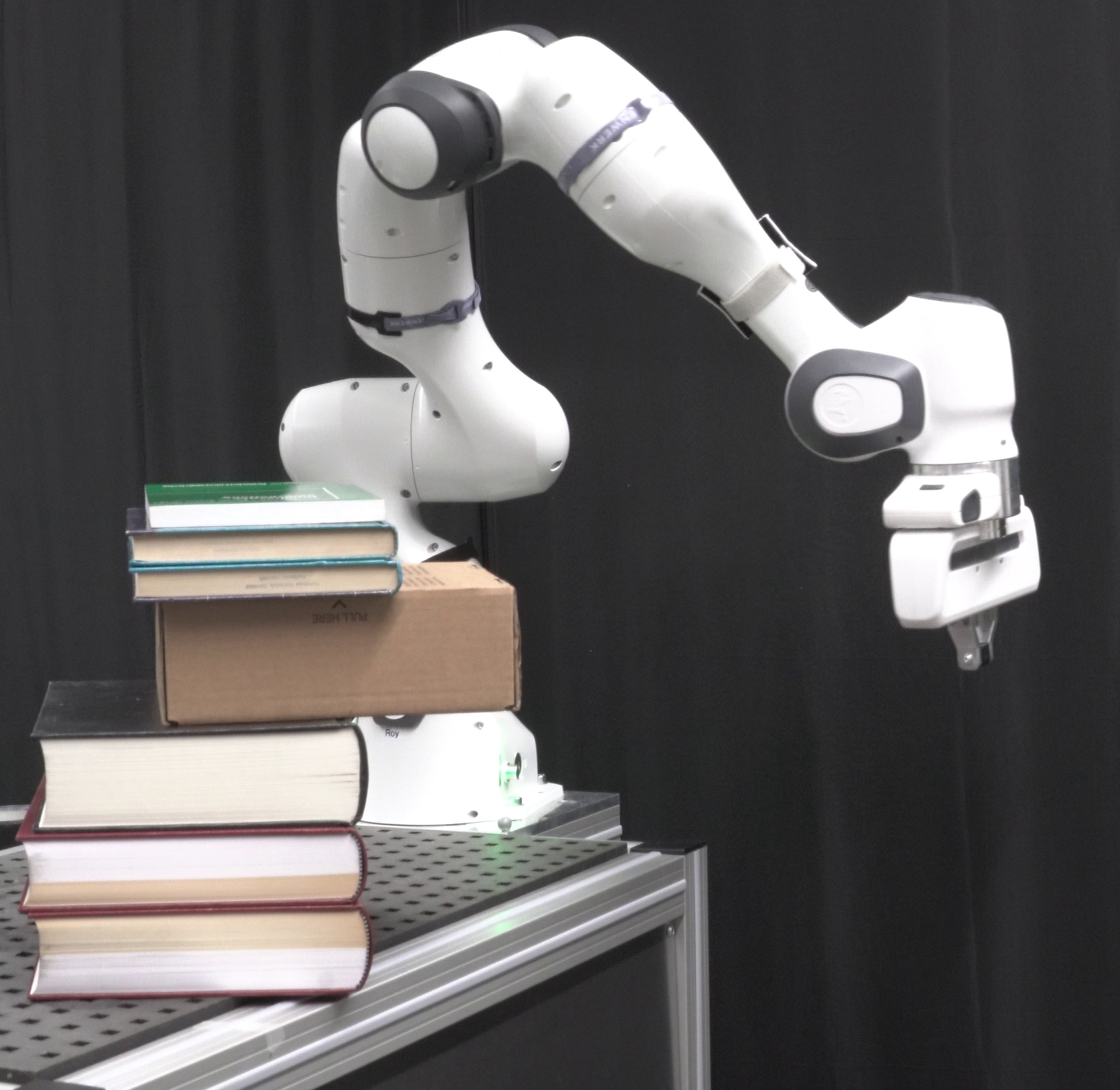}
\hspace{0.02\linewidth}  
\includegraphics[width=0.45\linewidth,height=3.6cm]{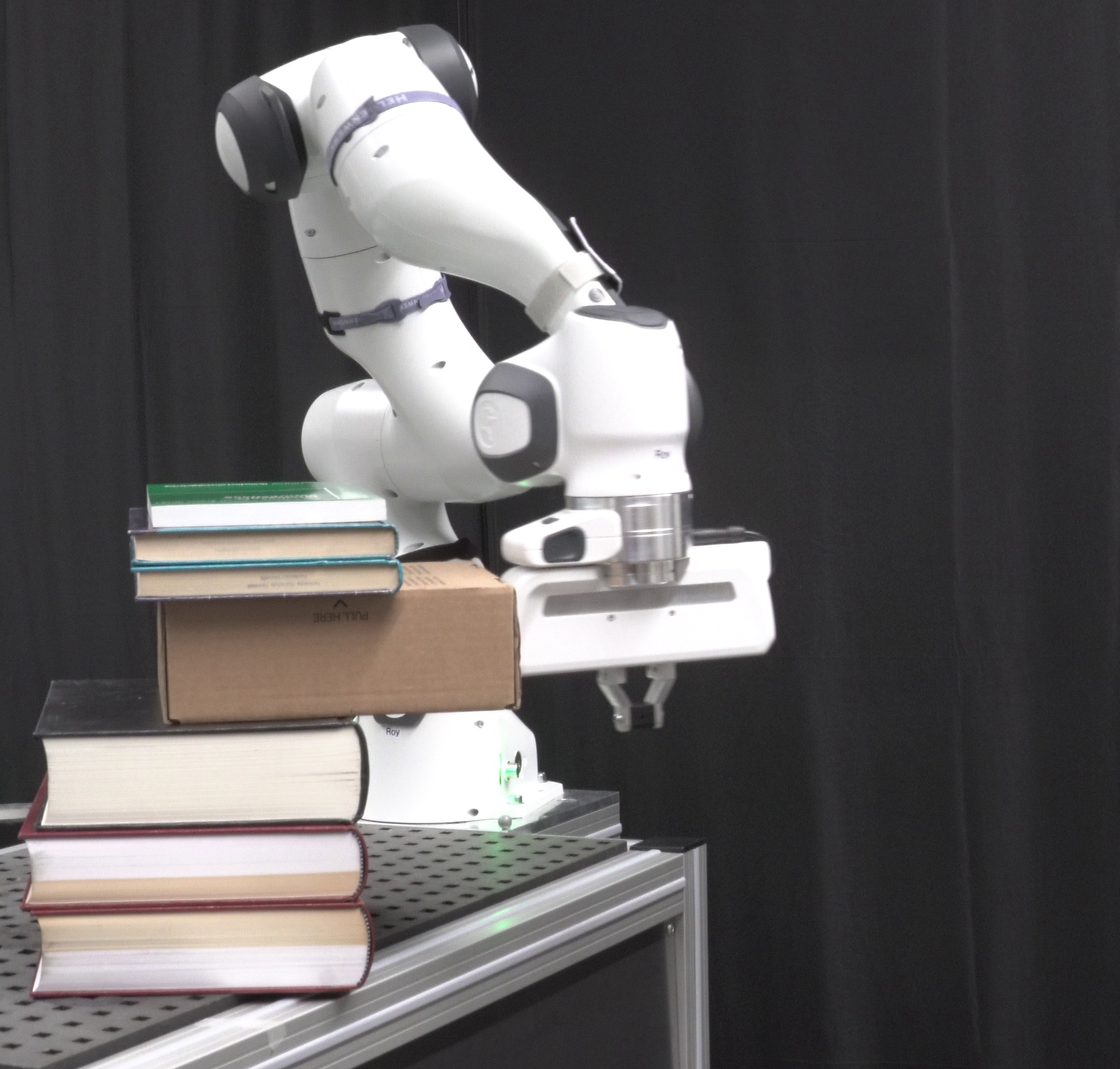}
\caption{The real FR3 Robot during a collision with a box.}
\label{Fig:05}
\end{center}
\end{figure}

\begin{figure*}[ht]
\centering
\begin{minipage}[t]{0.3\textwidth}
  \centering
  \textbf{Scenario C.1.a:~Joint torque}
  \end{minipage}%
\vspace{-0.025em}
\begin{minipage}[t]{0.3\textwidth}
  \centering
  \textbf{Scenario C.1.b:~Control effort}
\end{minipage}%
\vspace{-0.025em}
\begin{minipage}[t]{0.3\textwidth}
  \centering
  \textbf{Scenario C.1.c:~Tracking error}
\end{minipage}%
 \vspace{-0.025em}

\begin{minipage}[t]{0.3\textwidth}
  \centering
  \includegraphics[width=0.8\linewidth]{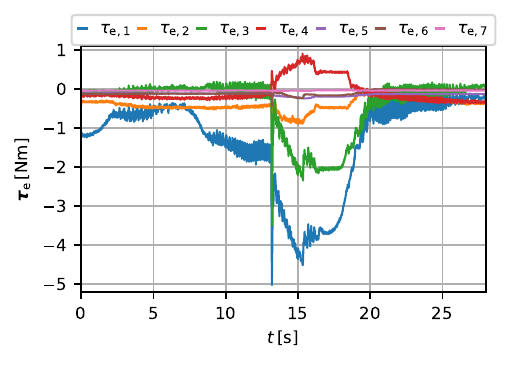}\\[-0.3em]
  \textbf{Scenario D.1.a:~Joint torque}\\[-0.1em]
\includegraphics[width=0.8\linewidth]{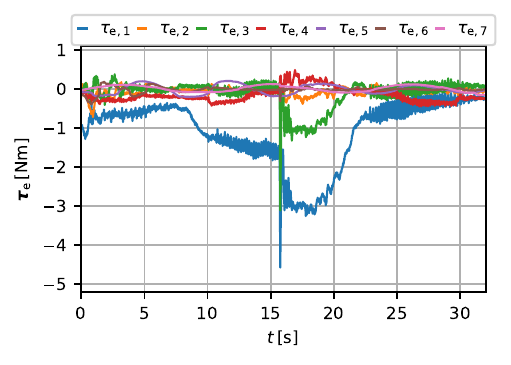}\\[-0.3em]
\end{minipage}%
 \vspace{-0.1em}
\begin{minipage}[t]{0.3\textwidth}
  \centering
  \includegraphics[width=0.8\linewidth]{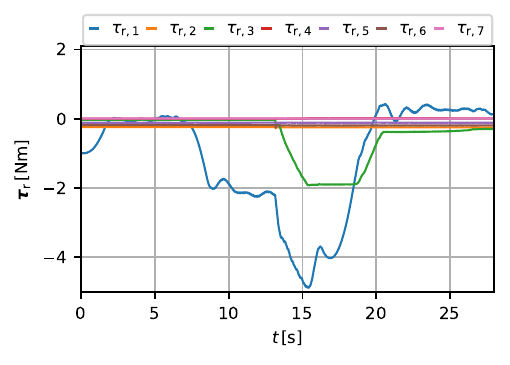}\\[-0.3em]
  \textbf{Scenario D.1.b:~Control effort}\\[-0.1em]
  \includegraphics[width=0.8\linewidth]{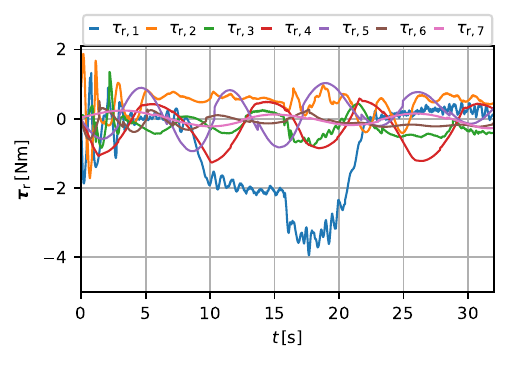}\\[-0.2em]
\end{minipage}%
\vspace{-0.1em}
\begin{minipage}[t]{0.3\textwidth}
  \centering
\includegraphics[width=0.8\linewidth]{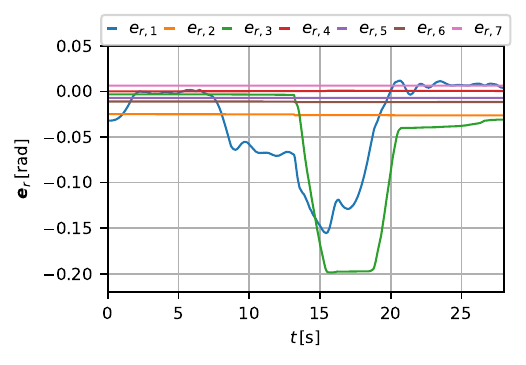}\\[-0.3em]
  \textbf{Scenario D.1.c:~Tracking error}\\[-0.1em]
  \includegraphics[width=0.8\linewidth]{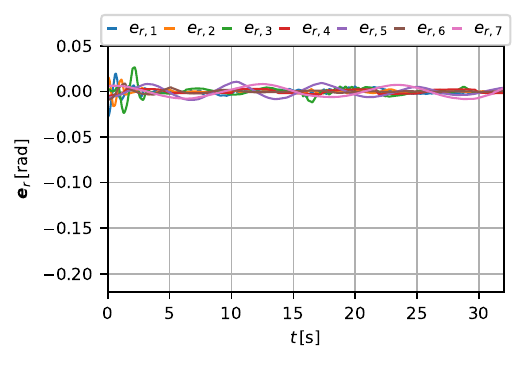}\\[-0.2em]
\end{minipage}
\vspace{-0.1em}

\caption{\textbf{(C)--(D):} Comparison of Scenario~C (PD control) and Scenario~D (proposed impedance control) for the real FR3 Robot. Columns show collision joint torque, control effort, and inner loop tracking error.}

\label{Fig:06}
\end{figure*}

\section{Conclusion and Future Perspectives}
\label{sec:Conclusion}
In this paper, we presented a robust adaptive backstepping impedance control strategy for robots. By addressing dynamic coupling and system uncertainties, the proposed method enables compliant and reliable interaction with unknown environments and unintended collisions. The integration of a Taylor series-based system approximation further enhances adaptability and robustness. Simulation results on a robotic arm mounted on a mobile base and experiments on the manipulator demonstrate the effectiveness and superiority of the proposed approach compared to conventional non-interactive control methods, particularly in safety-critical scenarios. Future work could investigate systematic extensions toward learning-enhanced impedance control. In particular, hybrid model-data approaches could integrate learning components to compensate for unmodeled dynamics or time-varying uncertainties while maintaining stability and safety through the adaptive backstepping structure. Promising directions include incorporating lightweight learning modules, such as locally valid function approximators or modest policy adaptation modules to improve performance in highly unstructured environments without compromising real-time feasibility.


{\small
{\bf Acknowledgment:} This work was partially supported by Research Center for Systems and Technologies - SYSTEC and the Associate Laboratory Advanced Production and Intelligent Systems – ARISE (DOI 10.54499/LA/P/0112/2020) funded by Fundação para a Ciência e a Tecnologia, I.P./ MECI through the national funds. The first author was supported by a Ph.D. Scholarship, grant 2022.11470.BD from FCT, Portugal. This work was also supported by the Deutsche Forschungsgemeinschaft (German Research Foundation, DFG) under Germany’s Excellence Strategy (EXC 3066/1 “The Adaptive Mind”, Project No. 533717223).
}

                                                   







\appendix
\section{}\label{app:lemmas}
\begin{lem} \label{lemma:inequality2}
Suppose \(a\), \(b\), and \(c\) are positive constants, and let \(q_1\) and \(q_2\) be arbitrary real numbers. Then, the following inequality is satisfied:
\begin{equation}
\lvert q_1 \rvert^a\lvert q_2 \rvert^b
\leq \frac{a}{a+b}p\lvert q_1 \rvert^{a+b}+\frac{b}{a+b}p^{-a/b}\lvert q_2 \rvert^{a+b}
\end{equation}
This result is presented in~\citep{qian2001non}.
\end{lem}

\begin{lem} 
\label{lemma:inequality3}
Assume that \(p_1, p_2, \ldots, p_n\) are positive real numbers and that \(l\) satisfies \(0 < l < 1\). Then, the following inequality holds:
\begin{equation}
(p_1+p_2+\dots+p_n)^l\leq p_1^l + p_2^l +\dots+p_n^l
\end{equation}
This statement is supported by~\citep{yu2005continuous}.
\end{lem}

\end{document}